\documentclass[11pt]{article}
\usepackage[margin=1in]{geometry}

\usepackage{amsmath}
\usepackage{amsthm}
\usepackage{amssymb}
\usepackage{algorithm}
\usepackage{color}
\usepackage[english]{babel}
\usepackage{graphicx}
\usepackage{grffile}
\usepackage{wrapfig,epsfig}
\usepackage{epstopdf}
\usepackage{url}
\usepackage{color}
\usepackage{epstopdf}
\usepackage{algpseudocode}
\usepackage[T1]{fontenc}
\usepackage{bbm}
\usepackage{comment}
\usepackage{dsfont}
\usepackage{thm-restate}
\usepackage{bm}
\usepackage{tcolorbox}
\usepackage{subcaption}
\usepackage{enumitem}

\usepackage{mathtools}

\newtheorem{theorem}{Theorem}[section]
 
\newtheorem{lemma}[theorem]{Lemma}

\newtheorem{conjecture}[theorem]{Conjecture}

\newtheorem{remark}[theorem]{Remark}

\newcommand{\eps}{\epsilon}
\DeclareMathOperator{\poly}{poly}
\DeclareMathOperator{\init}{init}
\DeclareMathOperator{\unif}{unif}
\newcommand{\mA}{\mathcal{A}}
\newcommand{\mS}{\mathcal{S}}
\newcommand{\mC}{\mathcal{C}}
\newcommand{\mB}{\mathcal{B}}
\newcommand{\wt}{\widetilde}
\newcommand{\wh}{\widehat}
\newcommand{\mr}{\mathsf{r}}

\newcommand{\me}{\mathsf{e}}
\newcommand{\mb}{\mathsf{b}}
\newcommand{\mt}{\mathsf{t}}
\newcommand{\mmp}{\mathsf{p}}
\newcommand{\SETH}{\mathsf{SETH}}
\newcommand{\mip}{\textsc{Max-IP}}

\newcommand{\mH}{\mathcal{H}}

\DeclareMathOperator*{\E}{{\mathbb{E}}}

\newcommand{\Binghui}[1]{{\color{blue}[Binghui: #1]}}

\title{The complexity of non-stationary reinforcement learning}
\author{  
 Christos Papadimitriou\\ Columbia University\\  \texttt{christos@columbia.edu} \and
 Binghui Peng \\ Columbia University \\ \texttt{bp2601@columbia.edu} 
}

\begin{document}
\maketitle

\begin{abstract}
The problem of continual learning in the domain of reinforcement learning, often called non-stationary reinforcement learning, has been identified as an important challenge to the application of reinforcement learning. We prove a worst-case complexity result, which we believe captures this challenge: Modifying the probabilities or the reward of a single state-action pair in a reinforcement learning problem requires an amount of time almost as large as the number of states  in order to keep the value function up to date, unless the strong exponential time hypothesis (SETH) is false; SETH is a widely accepted strengthening of the P $\neq$ NP conjecture.  Recall that the number of states in current applications of reinforcement learning is typically astronomical.  In contrast, we show that just {\em adding} a new state-action pair is considerably easier to implement.
\end{abstract}

\setcounter{page}{0}
\thispagestyle{empty}
\newpage

\section{Introduction}
\label{sec:intro}
Reinforcement learning (RL) \cite{sutton2018reinforcement}, the branch of machine learning seeking to create machines that react to a changing environment so as to maximize long-term utility, has recently seen tremendous advances through deep learning \cite{silver2017mastering,silver2018general}, as well as a vast expansion of its applicability and reach to many application domains, including board games, robotics, self-driving cars, control, and many more. As with most aspects of deep learning, one of the most important current challenges in deep RL lies in handling situations in which the model undergoes changes.  Variably called {\em non-stationary RL, continual RL, multi-task RL, or life-long RL}, the problem of enabling RL to react effectively and gracefully to sequences of changes in the underlying Markov model has been identified as an important open problem in practice, see the prior work subsection for many references, and \cite{khetarpal2022towards} for a recent survey of the challenge and the available remedies.

When it becomes clear that a particular computational problem is difficult, the field of {\em computational complexity} \cite{papadimitriou1998combinatorial,papadimitriou2003computational,arora2009computational} comes into play: the search for mathematical obstacles to the efficient solution of problems.  The identification of such obstacles is often informative about the kinds of remedies one needs to apply to the problem.  As far as we can tell, the computational complexity of non-stationary RL (NSRL) has not been explored in the past; in contrast, see \cite{chen2022memory} for an example of recent progress in identifying complexity obstacles in continual learning of {\em classification} tasks.

{\em In this paper, we initiate the analysis of NSRL from the standpoint of computational complexity.} We consider finite horizon MDPs --- it is easy to see that our results can be extended very easily to infinite horizon MDPs.  We ask the following question:  Suppose that we have already solved a finite-horizon MDP, and that the MDP changes in some small way; how difficult is it to modify the solution?  If the solution we want to update is an explicit mapping from states to actions, then it is not hard to see that this is hopeless: a small local change can cause a large proportion of the values of this map to change\footnote{For example, consider the extreme example where a change in an action increases the value of the next state, and this in turn changes the optimum actions in almost all other states.}.
However, recall that deep RL is not about computing explicitly the optimum solution of the problem; it is about maintaining an implicit representation of a good {\em approximation} of the optimum solution. 
An efficient NSRL algorithm only needs to update the value or policy efficiently when visiting the state.
Our results address precisely this aspect of the difficulty.  

We consider elementary local changes to the RL problem, which we believe capture well the nature of the NSRL problem: We choose a state-action pair  and we modify somehow its parameters: the reward, and the transition probability distribution.  Our results hold for the most elementary possible change:  We only modify two transition probabilities in this state-action pair.  (Notice that it is impossible to modify only one probability in a distribution...)  We prove that, under widely accepted complexity assumptions to be explained soon, the amount of computation needed to update an $\epsilon$-optimal value approximation in the face of such an elementary change is, in the worst case, comparable to the number of states (the precise result is stated below).  Since in the problems currently solved by deep RL the number of states of the underlying MDP is typically astronomical, such a prediction is bad indeed --- it means that we essentially have to start all over because of a small change.   Now, in deep learning we know well that a worst-case result is never the last word on the difficulty of a problem.  However, we believe that an alarming worst-case result, established for an aspect of the problem which has been identified in practice to be a challenge, is a warning sign which may yield valuable hints about the corrective action that needs to be taken in order to overcome the current bottleneck.

We complement this lower bound with a positive result for a different kind of change: {\em adding a new} action to a state.  It turns out that this is a simpler problem, and an $\epsilon$-approximate solution can be updated in time polynomial in $\tfrac{1}{\eps}$ and the horizon.  



\subsection*{Related work}

Non-stationary MDPs have been studied extensively in recent years from the point of view of dynamic regret \cite{auer2008near,dick2014online,ortner2020variational,cheung2020reinforcement,zhounonstationary,li2019online,touati2020efficient,wei2021non,domingues2021kernel,mao2021near}; In \cite{mao2021near} an algorithm with total regret $\wt{O}(S^{1/3}A^{1/3}\Delta^{1/3}HT^{2/3})$ is provided, where $T$ is the total number of iteration, $\Delta$ is the variational budget that measures the total change of MDP.  
Another line of work focuses on the statistical problem of detecting the changes in the environment, see \cite{da2006dealing,banerjee2017quickest,padakandla2020reinforcement,ornik2021learning} , and \cite{padakandla2021survey,khetarpal2022towards} for recent surveys; in particular,  \cite{padakandla2021survey} mentions the computational difficulty of the change problem addressed in this paper.   Several approaches to NSRL  --- e.g \cite{wei2021non,mao2021near} --- resort to {\em restarting} the learning process if enough change has accumulated; our results suggest that, indeed, restarting may be preferable to updating. Additional literature can be found at Appendix \ref{sec:relate-app}.

\subsection*{A brief overview of the main result}
Our main result (Theorem \ref{thm:fully}) states that, in the worst case, an elementary change in an MDP --- just updating two transition probabilities in one action at one state of the MDP --- requires time $(SAH)^{1-o(1)}$, where $S$ is the number of states, $A$ is the number of action and $H$ is the horizon. The proof is based on the Strong Exponential Time Hypothesis ($\SETH$), which is a central conjecture in complexity, a refinement of $\mathsf{P}\neq\mathsf{NP}$.  $\SETH$ has many applications in graph algorithms \cite{roditty2013fast,abboud2014popular,backurs2018towards,li2021settling,dalirrooyfard2022hardness}, edit distance \cite{backurs2015edit}, nearest neighbor search \cite{rubinstein2018hardness}, kernel estimation  \cite{charikar17hash,alman20algorithm} and many other domain; see \cite{rubinstein2019seth} for a comprehensive survey.  $\SETH$ states that, if the $k$-SAT problem (the Boolean satisfiability problem when each clause contains at most $k$ literals) can be solved in time $O(2^{c_kn})$, then the limit of $c_k$ as $k$ grows is one.  Our work is based on the important result of \cite{abboud2017distributed} on the hardness, under $\SETH$, of approximating the bichromatic Maximum Inner Product ($\mip$) problem. Subsequent work has improved the approximation parameter \cite{rubinstein2018hardness,chen2020hardness} and applied the technique to the Dynamic Coverage problem \cite{abboud2019dynamic,peng2021dynamic}.

We reduce from the $\mip$ problem, where we are given two collections of sets $B_1, \ldots, B_n$ and $C_1, \ldots, C_n$, over a small universe $[m]$ with $m = n^{o(1)}$. It is known from \cite{abboud2017distributed} that it is hard to distinguish between the following two scenaria: (a) $B_i \subseteq C_j$ for some $i, j \in [n]$, and (b) $|B_i \cap C_j| \leq |C_j|/2^{\log(n)^{1-o(1)}}$ for all $i, j \in [n]$.  That is, it is hard to tell the difference between the case of a complete containment and the case of tiny intersections.  The first step of our reduction is to construct a finite-horizon MDP such that the state of the first step ($h=1$) corresponds to the sets $B_1, \ldots, B_n$ and the state of the second step ($h=2$) corresponds to the universe $[m]$. The state of the second step has either high reward or low reward, depending on the time $t$.
By applying a sequence of changes to the state-action transition in the second step, based on the structure of the sets $C_1, \ldots, C_n$, one obtains a reduction from $\mip$ establishing a lower bound of $S^{2-o(1)}$ for this sequence.  However, since this sequence is of length $S^{1+o(1)}$ (because of the size of the $C_j$ sets), we obtain an $\Omega(S^{1-o(1)})$ amortized  
lower bound for each step of the sequence, and this complete the reduction to the NSRL problem. 

The construction so far yields an approximation $\eps$ that is very small (about $S^{-o(1)}$).  We need a second stage of our construction to amplify $\eps$ to some constant such as $0.1$. This is achieved by stacking multiple layers of the basic construction outlined above. 
Finally, by spreading the state-actions across multiple steps, we improve the lower bound to $\Omega((SAH)^{1 - o(1)})$.

The complete proof can be found at Section \ref{sec:fully}.  

\section{Preliminary Definitions}
\label{sec:preliminary}

Here we shall define non-stationary MDPs.   Let $\mS$ be a state space ($|\mS| = S$), $\mA$ an action space ($|\mA| = A$), $H \in \mathbb{Z}_{+}$ the planning horizon.  Next let $T \in \mathbb{Z}_{+}$ be the number of {\em rounds:} The intention is that the MDP will repeated $T$ times, with action parameters changed between rounds.  

A {\em non-stationary finite horizon MDP} is a set of $T$ MDPs $(\{\mS_h, \mA_h, P_{t, h},r_{t, h}\}_{t\in [T], h\in [H]}, s_{\init})$.  $\mS_h \subseteq \mS$ is the state space and $\mA_h \subseteq \mA$ is the action space at the $h$-th step ($h\in [H]$), and
$P_{t, h}: \mS_h\times \mA_h \rightarrow \Delta_{\mS_{h+1}}$ is the transition function, where $\Delta_{\mS_h}$ is the set of all probability distributions over $\mS_h$, and $r_{t, h}: \mS_h\times \mA_h \rightarrow [0,1]$ is the reward function at the $h$-th step of the $t$-th round ($h \in [H], t\in [T]$).
We use $s_{\init} \in \mS_1$ to denote the initial state.

We focus on deterministic non-stationary policies $\pi= (\pi_1, \ldots, \pi_T)$, though our results apply for randomized policies as well.
Let $\pi_{t} = (\pi_{t,1}, \ldots, \pi_{t, H})$ be the policy of the $t$-th round ($t\in [T]$) and $\pi_{t, h} : \mS_h\rightarrow \mA_h$ ($h\in [H]$) be the decision at the $h$-th step.
Given a policy $\pi$, the $Q$-value of a state-action pair $(s, a) \in \mS_h\times \mA_h$ at the $t$-round can be determined 
\[
Q_{t, h}^{\pi_t}(s, a) = r_{t, h}(s, a) + \E\left[ \sum_{\ell=h+1}^{H} r_{t, h}(s_{t,\ell}, \pi_{t,\ell}(s_{t,\ell})) \mid s_{t,h}= s, a_{t, h}= a \right] \quad \forall s \in \mS_{h} , a \in \mA_{h}
\]
and the $V$-value 
\[
V_{t, h}^{\pi_t}(s) = \E\left[ \sum_{\ell=h}^{H} r_{t, h}(s_{t,\ell}, \pi_{t,\ell}(s_{t,\ell})) \mid s_{t,h}= s \right] \quad \forall s \in \mS_{t, h}.
\]

Let $\pi_t^{*}$ be the optimal policy at the $t$-round, and $Q^{*}_t$, $V_{t}^{*}$ be its $Q$-value and $V$-value.
The goal is to maintain an $\eps$-approximated value function. In particular, we require the algorithm to maintain an $\eps$-approximated estimation $V_t$, of the initial state $s_{\init}$, such that for all rounds $t\in [T]$,
\begin{align*}
\left|V_t - V_{t,1}^{*}(s_{\init})\right| \leq \eps.
\end{align*}

\paragraph{Updates.} 
All $T$ MDPs of our definition must be solved, one after the other, despite the fact that their parameters change from one round to the next.  The updates are meant to be extremely simple and local: For the $t$-th update, an adversary picks an arbitrary state-action pair $(s_h, a_h) \in \mS_h \times \mA_h$, and changes the transition function from $P_{t-1,h}(s_h, a_h)$ to $P_{t,h}(s_h, a_h)$ and the reward from $r_{t-1,h}(s_h, a_h)$ to $r_{t,h}(s_h, a_h)$.  It also changes the transition function from $P_{t-1,h}(s_h, a_h)$ to $P_{t,h}(s_h, a_h)$, such that these two distributions differ in {\em exactly two states.}  That is, the change in the distribution is the smallest kind imaginable:  {\em Two next states are chosen, and the probability mass of the first is transferred to the second} --- obviously, two discrete distributions cannot differ in exactly one probability.

\begin{remark}
Implementing an elementary change of this kind takes constant time: If the distribution is represented in a tabular form, the two entries of the table are changed. 
It holds similarly when the MDP is accessed via a sampling oracle (a.k.a. the generative model), all one has to do is change the output states.
\end{remark}

\begin{remark}
Notice that the kind of changes we consider is the simplest possible, and yet a sequence of such changes can simulate any desirable change.  Hence, by showing in the next section that even these changes are computationally intractable, we establish that NSRL is intractable. 
\end{remark}

\paragraph{Incremental action change.}
We also consider a different type of NSRL, where the MDP changes only through the introduction of a new action.\footnote{Note the introduction of a new {\em state} can be achieved through a sequence of action additions.}
The setup is similar to NSRL: we assume the initial MDP has $S$ states, $H$ steps but the action set is empty. Then in each round $t \in [T]$, a new state-action pair $(s_h, a_h)$ is added to the MDP, together with its transition probability $P_{h}(s_h, a_h)$ and reward $r_{h}(s_h, a_h)$.
Note the crucial difference with NSRL is that the there is no change occurs on any existing state-action pair.
There are a total of $T$ rounds, and therefore, $T$ state-action pairs at the end.
The incremental action model captures application scenario that involves explorations or expansion of environments (e.g. incremental training).

\section{Hardness of NSRL}
\label{sec:fully}

The main result is the following:
\begin{theorem}[Main result, hardness of NRSL]
\label{thm:fully}
Let $S, A, H, T$ be sufficiently large integers, the horizon $H \geq (SA)^{o(1)}$.
Then, unless $\SETH$ is false, there is no algorithm with amortized runtime $O((SAH)^{1-o(1)})$ per update that can approximate the optimal value of a non-stationary MDP over a sequence of $T$ updates. In particular, any algorithm with better runtime fails to distinguish between these two cases:
\begin{itemize}
\item The optimal policy has value at least $\frac{H}{4}$ at some round $t \in [T]$;
\item The optimal policy has value at most $\frac{H}{100}$ for all $T$ rounds.
\end{itemize}
\end{theorem}

Our result is based on the widely accepted Strong Exponential Time Hypothesis ($\SETH$). 
\begin{conjecture}[Strong Exponential Time Hypothesis ($\SETH$), \cite{impagliazzo2001complexity}]
For any $\eps > 0$, there exists $k \geq 3$ such that the $k$-SAT problem on $n$ variables cannot be solved in time $O(2^{(1-\eps)n})$.
\end{conjecture}

Note that $\SETH$ is stronger than the $\mathsf{P}\neq\mathsf{NP}$ assumption, a strengthening that allows the proof of {\em polynomial} lower bounds on problems that have a polynomial-time algorithm --- such as NSRL.

The starting point of our reduction is the following hardness result for the Bichromatic Maximum Inner Product (\mip) problem, whose proof is based on the machinery of distributed PCP.
\begin{theorem}[Bichromatic Maximum Inner Product (\mip) \cite{abboud2017distributed}] 
\label{thm:arw}
Let $\gamma > 0$ be any constant, and let $n \in \mathbb{Z}_{+}$, $m = n^{o(1)}$, $w = 2^{(\log(n))^{1-o(1)}}$. Given two collections of sets $\mB = \{B_1, \ldots, B_n\}$ and $\mC = \{C_1, \ldots, C_n\}$ over universe $[m]$, satisfying $|B_1| = \cdots =|B_n| = b$ and $|C_1| = \cdots =|C_n| = c$ for some $b, c \in [m]$. Unless $\SETH$ is false, no algorithm can distinguish the following two cases in time $O(n^{2-\gamma})$:
\begin{itemize}
\item {\bf YES instance. } There exists two sets $B \in \mB$, $C \in \mC$ such that $C \subseteq B$;
\item {\bf NO instance. } For every $B \in \mB$ and $C \in \mC$, $|B\cap C| \leq c/w$.
\end{itemize}
\end{theorem}

\paragraph{Parameters.} We reduce $\mip$ to NSRL.
For any sufficiently large parameters $S, A, H, T$, let 
\[
n = T^{1/2-o(1)}\cdot (SAH)^{1/2} \quad \text{and} \quad m = n^{o(1)} 
\]
be the input parameters of $\mip$. Given a $\mip$ instance with sets $B_1,\ldots, B_n$ and $C_{1}, \ldots, C_n$ over a ground set $[m]$, recall $b, c \in [m]$ are the size of set $\{B_i\}_{i\in [n]}$ and $\{C_i\}_{i \in [n]}$ . Let
\[
L = \lceil b/c \rceil \quad \text{and} \quad N = \frac{SAH}{16L(\log_2(S)+2)}.
\]
We shall divide $\{B_i\}_{i \in [n]}$ into $K = n/N$ batches and each batch contains $N$ sets. That is, $\{B_i\}_{i \in [n]} = \{B_{k, \nu}\}_{k\in [K], \nu\in [N]}$. In the proof, we assume the total number of updates $SAH \leq T \leq \poly(SAH)$, i.e., it is polynomially bounded.

\subsection{Construction of a hard instance} 
We first describe the MDP at the initial stage ($t= 0$), with state space $\{\mS_h\}_{h \in [H]}$, action space $\{\mA_h\}_{h \in [H]}$, transition function $\{P_h\}_{h\in [H]}$ and reward function $\{r_h\}_{h\in [H]}$. A (simplified) illustration can be found at Figure \ref{fig:hardness}. We omit the subscript of $t= 0$ for simplicity. 
\begin{figure}[!htbp]
\includegraphics[width=\textwidth]{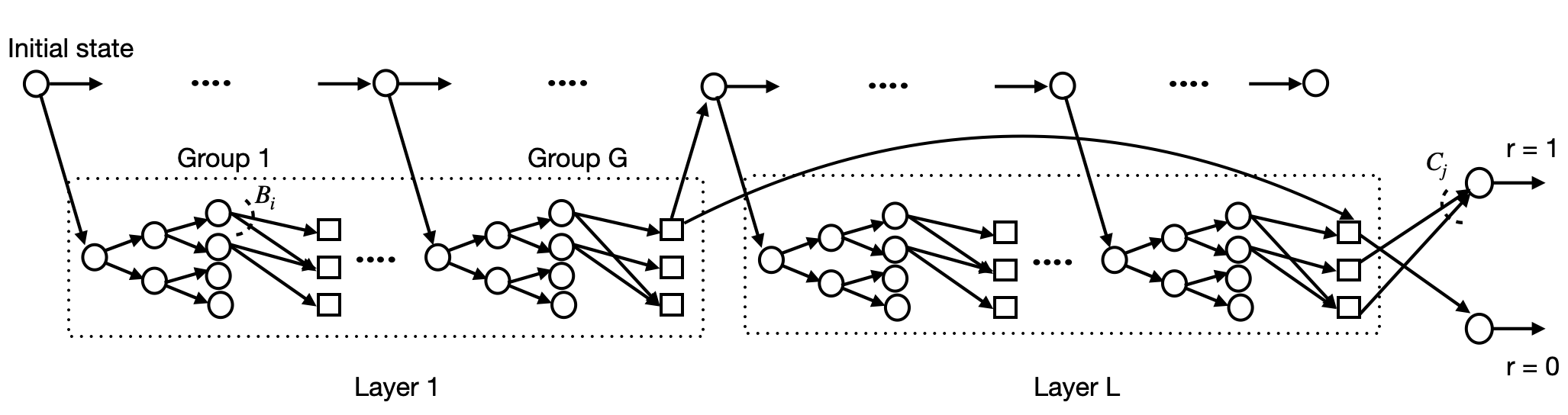}
\caption{A snapshot of the hard instance}
\label{fig:hardness}
\end{figure}

\paragraph{Horizon.} We divide the entire horizon into two phases 
\[
[H] = \mH_1 \cup \mH_2, \quad \text{where} \quad \mH_1 = [H/2]\quad \text{and}\quad \mH_2 = \left[H/2: H\right].
\]
The second phase is relatively simple and involves only two terminal states that provide rewards. The first phase is more involved and determines the destination state.

The first phase contains $L$ layers, and each layer contains $H/2L$ steps
\[
\mH_1 = \mH_{1,1} \cup \cdots \cup \mH_{1, L}, \quad \text{where} \quad \mH_{1, \ell} = \left[(\ell-1) \cdot \frac{H}{2L} + 1: \ell \cdot \frac{H}{2L} \right] \,\forall \ell \in [L].
\]
The layers are used for amplifying the difference between good and bad policies. The structure of the MDP is for identical for each layer, except the last step at the last layer.

For each layer $\ell \in [L]$, we further divide it into $G := \frac{H}{2L(\log_2(S)+2)}$ groups, and each group contains $\log_2(S)+2$ steps,
\[
\mH_{1, \ell} = \mH_{1, \ell, 1} \cup \cdots \cup \mH_{1, \ell, G}
\]
where
\[ 
\mH_{1, \ell,g} = \left[(\ell-1) \cdot \frac{H}{2L} + (g-1)(\log_2(S)+2) + 1: (\ell-1) \cdot \frac{H}{2L} + g \cdot (\log_2(S) + 2) \right] \,\forall g \in [G]. 
\]

 We set $h(\ell, g, \tau) := (\ell-1)(H/2L) + (g-1)(\log_2(S)+2) + \tau$ be the $\tau$-step, at the $g$-th group of the $\ell$-th layer, where $\tau \in [\log_2(S)+2],g\in [G], \ell\in [L]$. 
For simplicity, we also write $h(\ell, g) = h(\ell, g, \log_2(S)+2)$ and $h(\ell) = h(\ell, G)$ be the last step of each group and layer.

\paragraph{States.} There are five types of states: terminal states, element states, set states, routing states and the pivotal state.
\begin{itemize}
\item {\bf Terminal states.} There are two terminal states $s^{\mt}_1$ and $s^{\mt}_2$, and they appear at every steps $h \in [H]$. We use $s^{\mt}_{h,1}, s^{\mt}_{h,2}$ to denote the terminal states at $\mS_h$.
\item {\bf Element states.} There are $m$ element states $\{s^{\me}_u\}_{u\in [m]}$ that appear at every step $h \in \mH_1$ of phase one. We use $s^{\me}_{h,u}$ to denote the $u$-th element state at $\mS_h$.
\item {\bf Set states.} There are $S/4$ set states $\{s^{\mb}_{i}\}_{i \in [S/4]}$. The set states only appear on the second last step of each group $\mH_{\ell, g}$. In particular, for each layer $\ell \in [L]$, group $g \in [G]$, let $s^{\mb}_{h(\ell, g)-1, i}$ denote the $i$-th ($i \in [S/4]$) set state at $\mS_{h(\ell, g)-1}$.
\item {\bf Pivotal state} There is one pivotal state $s^{\mmp}$ that appears at every step $h\in \mH_1$ of Phase 1, denoted as $s^{\mmp}_{h}$. The MDP start with the pivotal state, i.e., $s_{\init}:= s_{1}^{\mmp}$.
\item {\bf Routing states.} The routing states are used for reaching set states.  There $S/4$ routing states $\{s_{\alpha}^{\mr}\}_{\alpha \in [S/4]}$ that appear at the $[2:\log_2(S)]$-th step of each group. In particular, at layer $\ell\in [L]$, group $g\in [G]$, step $\tau \in [2:\log_2(S)]$, let $\{s^{\mr}_{h(\ell, g, \tau), \alpha}\}_{\alpha \in [1:2^{\tau-2}]}$ be the collection of routing states at $\mS_{h(\ell, g, \tau)}$.
\end{itemize}

The total number of possible states is at most $2 + m + S/4 + S/4 + 1 \leq S$.

\paragraph{Actions}
There are five types of actions. The terminal action $a^{\mt}$, the element actions $a^{\me}$, the set actions $\{a^{\mb}_j\}_{j \in [A/2]}$, the pivotal action $\{a^{\mmp}_1, a^{\mmp}_2\}$ and the routing actions $\{a^{\mr}_1, a^{\mr}_2\}$. 
The total number of action is at most $A/2 + 6 \leq A$, and we assume these actions appear at every step $h \in [H]$.

\paragraph{Reward}
The only state that returns non-zero reward is the terminal state $\{s^{\mt}_{h, 1}\}_{h \in \mH_2}$. Formally, we set 
\begin{align}
r_{h}(s, a) = 0 \quad \text{when} \quad h \in \mH_1 \quad \text{and} \quad 
r_{h}(s, a) = \left\{
\begin{matrix}
1 & s = s^{\mt}_{h, 1} \\
0 & \text{otherwise}
\end{matrix}
\right.
\quad \text{when} \quad h \in \mH_2. \label{eq:reward}
\end{align}

\paragraph{Transitions} We next specify the transition probability of the initial MDP.

{\bf (a) Terminal states.} The transition of terminal states is deterministic and always keeps the state terminal, that is
\begin{align}
P_h(s_{h, 1}^{\mt}, a) = \mathbf{1}\{s_{h+1, 1}^{\mt}\} \quad \text{and} \quad P_h(s_{h, 2}^{\mt}, a) = \mathbf{1}\{s_{h+1, 2}^{\mt}\} \quad \forall h\in [H-1], a\in \mA. \label{eq:terminal-transition}
\end{align}
Here we use $\mathbf{1}\{s\} \in \Delta_{\mS_{h+1}}$ to denote the one-hot vector that is $1$ at state $s$ and $0$ otherwise. Combining with the definition of reward functions, the MDP guarantees that a policy receives $H/2$ reward once it goes to the first terminal state $s_{h, 1}^{\mt}$ at some step $h\in \mH_2$. Meanwhile, it receives $0$ reward if it ever goes to the second terminal state $s_{h, 2}^{\mt}$.

{\bf (b) Element states.} At step $h < H/2$, for any element $u\in [m]$, the transition function of $s_{h, u}^{\me}$ equals
\begin{align}
P_{h}(s^{\me}_{h, u}, a^{\me}) = 
\left\{
\begin{matrix}
\mathbf{1}\{s^{\mmp}_{h + 1}\} & h = h(\ell) \text{ for some } \ell \in [L-1]\\
\mathbf{1}\{s^{\me}_{h + 1, u}\} & \text{otherwise} 
\end{matrix}
\right. \label{eq:element-transition1}
\end{align}
and 
\begin{align}
P_{h}(s^{\me}_{h, u}, a) = \mathbf{1}\{s^{\me}_{h + 1, u}\}, \quad \forall a\in \mA \backslash \{a^{\me}\}. \label{eq:element-transition2}
\end{align}
That is, the element state $s_{h, u}^{\me}$ always stays on itself, except at the end of each layer $\ell \in [L]$, it can go to the pivotal state.

At the end of the first phase, the transition of element state is determined by the set $C$. 
In the initialization stage ($t=0$), let $C_{0} \subseteq [m]$ be an arbitrary set of size $c$ and it would be replace later, let
\[
P_{H/2}(s^{\me}_{H/2, u}, a^{\me}) = \left\{
\begin{matrix}
\mathbf{1}\{s_{H/2+1,1}^{\mt}\} & u \in C_{0}\\
\mathbf{1}\{s_{H/2+1,2}^{\mt}\} & u \notin C_{0}\\
\end{matrix}
\right.
\]
and 
\[
P_{H/2}(s^{\me}_{H/2, u}, a) = \mathbf{1}\{s_{H/2+1,2}^{\mt}\}, \quad \forall a\in \mA \backslash \{a^{\me}\}.
\]
That is, if the element $u \in C_{0}$, then it can go to a high reward terminal state $s_{H/2+1, 1}^{\mt}$; otherwise it goes to the no-reward terminal $s_{H/2+1, 2}^{\mt}$.
Looking ahead, we would update the state-action pairs $\{(s_{H/2, u}^{\me}, a^{\me})\}_{u \in [m]}$ according to sets $\{C_{i}\}_{i \in [n]}$ periodically.

{\bf (c) Set states.} The transition function of set states is determined by the sets $\{B_{k, \nu}\}_{k \in [K],\nu\in [N]}$.
In the initialization stage ($t= 0$), let $\{B_{0, \nu}\}_{\nu \in [N]}$ be arbitrary sets of size $b$ and they would be replaced later in the update sequence.
Recall that a set state would appear at the second last step of a group $\mH_{\ell, g}$, for some layer $\ell \in [L]$ and group $g \in [G]$.
Let 
\[
N(g, i, j) := (g-1)(S/4)(A/2) + (i-1)(A/2) + j,
\]
and therefore, 
\[
\{N(g, i, j): g\in [G], i \in [S/4], j \in [A/2]\} = [N].
\]

The transition function of state-action pair $(s_{h(\ell, g)-1, i}^{\mb}, a_{j}^{\mb})$ equals
\begin{align}
P_{h(\ell, g)-1}(s^{\mb}_{h(\ell, g)-1, i}, a^{\mb}_{j}) = \unif(s^{\me}_{h(\ell, g), u}: u \in B_{0, N(g, i, j)})
\quad \forall  g \in [G], i \in [S/4], j \in [A/2]. \label{eq:state-transition}
\end{align} 
Here the RHS is the uniform distribution over the element states $s_{h(\ell, g), u}^{\me}$ for element $u \in B_{0, N(g, i, j)}$.
For the rest of actions, it goes to the no-reward terminal $s^{\mt}_{h(\ell, g), 2}$:
\begin{align*}
P_{h(\ell, g)-1}(s^{\mb}_{h(\ell, g)-1, i}, a) = \mathbf{1}\{s_{h(\ell, g),2}^{\mt}\} \quad \forall a \in \mA \backslash \{a^{\me}_{j}\}_{j \in [A/2]} 
\end{align*}

{\bf (d) Pivotal states.} The pivotal state $s^{\mmp}_{h}$ appears at every step $h \in \mH_1$, and for $h < H/2 - 1$, the transition function equals
\begin{align}
P_{h}(s^{\mmp}_{h}, a) = \left\{
\begin{matrix}
\mathbf{1}\{s^{\mr}_{h+1, 1}\} & a = a^{\mmp}, h = h(\ell, g,1)  \text{ for some } \ell\in [L], g\in [G]\\
\mathbf{1}\{s^{\mmp}_{h+1}\} & \text{otherwise}
\end{matrix}
\right. \label{eq:pivotal-transition}
\end{align}
That is, the pivotal state stays on itself, except at the first step of $\mH_{\ell, g}$, it could go to the routing state $s_{h(\ell, g,2),1}^{\mr}$.

At the $H/2$-th step, it goes to the no-reward terminal $s_{H/2+1,2}^{\mt}$,
\begin{align*}
P_{H/2}(s^{\mmp}_{H/2}, a)=\mathbf{1}\{s^{\mt}_{H/2+1, 2}\} \quad \forall a \in A. 
\end{align*}

{\bf (e) Routing states.} Recall $\{s^{\mr}_{h(\ell, g, \tau), \alpha}\}_{\alpha \in [1:2^{\tau-2}]}$ is the collection of routing states at the $\alpha$-th step ($\alpha \in [2:\log_2(S)]$), $g$-th group ($g\in [G]$) and $\ell$-th layer ($\ell \in [L]$).

When $\tau \in [2:\log_2(S) -1]$, the transition function equals
\begin{align}
P_{h(\ell, g, \tau)}(s^{\mr}_{h(\ell, g, \tau), \alpha}, a) = \left\{
\begin{matrix}
\mathsf{1}\{s^{\mr}_{h(\ell, g, \tau+1), 2\alpha-1}\} & a = a_1^{\mr}\\
\mathsf{1}\{s^{\mr}_{h(\ell, g, \tau+1), 2\alpha}\} & a = a_2^{\mr}\\
\mathsf{1}\{s^{\mt}_{h(\ell, g, \tau+1), 2}\} & \text{otherwise}
\end{matrix}
\right., \quad \forall \alpha \in [2^{\tau-2}]. \label{eq:routing1}
\end{align}
In other words, the routing state $s^{\mr}_{h(\ell, g, \tau), \alpha}$ goes to either $s^{\mr}_{h(\ell, g, \tau+1), 2\alpha-1}$ or $s^{\mr}_{h(\ell, g, \tau+1), 2\alpha-1}$, depending on the choice of actions (unless it goes to the no-reward terminal $s^{\mr}_{h(\ell, g, \tau+1),2}$). 

When $\tau = \log_2(S)$, the routing state $s^{\mr}_{h(\ell, g, \log_2(S)), \alpha}$  goes to the set state $s^{\mb}_{h(\ell, g)-1, \alpha}$ ($\alpha \in [S/4]$), that is,
\begin{align}
P_{h(\ell, g, \log_2(S))}(s^{\mr}_{h(\ell, g, \log_2(S)), \alpha}, a) =  s^{\mb}_{h(\ell, g)-1, \alpha}, \quad \forall \alpha \in [S/4], a \in \mA. \label{eq:routing2}
\end{align}
The entire transition of routing states within a group works like a binary search tree: it comes from the pivotal state and goes to one of the set states. We note that if $S \leq A$ the construction could be simplified: we can remove routing states and have a pivotal state directly go to set states. This completes the description of the initial MDP.

\paragraph{Update sequence.} We next specify the sequence of updates to the MDP. 
The sequence of updates is divided into $K = n/N$ stages, and each stage contains $n$-epochs.

At the beginning of each stage, the update occurs on the state-action pairs for set-states:
\[
\{(s^{\mb}_{h(\ell, g)-1,i}, a^{\mb}_j)\}_{\ell\in [L], g\in [G], i\in [S/4], j\in [A/2]}
\] 
Concretely, there is an initialization phase at the beginning of the $k$-th stage ($k\in [K]$). Let $t(k)\in [T]$ be the end of initiazation phase, and the nature sets
\begin{align*}
P_{t(k), h(\ell, g)-1}(s^{\mb}_{h(\ell, g)-1, i}, a^{\mb}_{j}) = \unif(s^{\me}_{h(\ell, g), u}: u \in B_{k, N(g, i, j)})
\quad \forall \ell \in [L],  g \in [G], i \in [S/4], j \in [A/2]. 
\end{align*}

Each stage contains $n$-epochs, and during each epoch, the update occurs on the state-action pairs $\{(s^{\me}_{H/2, u}, a^{\me})\}_{u \in [m]}$ of element state-action, in the $H/2$-th step. 
Let $t(k, \tau) \in [T]$ be the end of $k$-th ($k\in [K]$) stage and $\tau$-th ($\tau\in [n]$) epoch.
In the $\tau$-th epoch ($\tau \in [n]$), for each element $u \in [m]$, the transition function is updated to
\begin{align}
P_{t(k, \tau), H/2}(s^{\me}_{H/2, u}, a^{\me}) = \left\{
\begin{matrix}
\mathbf{1}\{s_{H/2+1,1}^{\mt}\} & u \in C_{\tau}\\
\mathbf{1}\{s_{H/2+1,2}^{\mt}\} & u \notin C_{\tau}
\end{matrix}. \label{eq:transition-new}
\right.
\end{align}

To count the total number of updates, there are $K = n/N$ stages. The initialization takes at most $O(SAHm)$ updates; there are $n$ epochs, and each epoch contains at most $2m$ updates.
Hence the total number of updates equals $(n/N) \cdot O(SAHm + 2mn) \approx T$.

\subsection{Analysis}
We now proceed to prove Theorem \ref{thm:fully}. For any stage $k \in [K]$ and epoch $\tau\in [n]$, we compute the $V$-value of the optimal policy. The proof can be found at the Appendix \ref{sec:fully-app}
\begin{lemma}[$V$-value, terminal states]
\label{lem:v-value-terminal}
At the end of stage $k \in [K]$ and epoch $t \in [n]$, for any step $h \in [H]$, the $V$-value of optimal policy at terminal states satisfies $V^{*}_{t(k, \tau), h}(s^{\mt}_{h, 1}) = \min\{H+1 - h, H/2\}$ and $V^{*}_{t(k, \tau), h}(s^{\mt}_{h, 2}) = 0$.
\end{lemma}

\begin{lemma}[$V$-value, element states] 
\label{lem:v-value-element}
At the end of stage $k \in [K]$ and epoch $\tau \in [n]$, for any layer $\ell \in [L]$ and any step $h \in \mH_{1, \ell}$ 
\begin{itemize}
\item For any element $u \in C_\tau$, $V^{*}_{t(k, \tau), h}(s^{\me}_{h, u}) = H/2$; and
\item For any element $u \notin C_\tau$, we have $V^{*}_{t(k, \tau), h}(s^{\me}_{h, u}) = V^{*}_{t(k, \tau), h(\ell)+1}(s^{\mmp}_{h(\ell)+1})$. 
\end{itemize}
Here we take $V_{t(k, \tau), H/2+1}(s^{\mmp}_{H/2+1}) := 0$.
\end{lemma}

\begin{lemma}[$V$-value, set states]
\label{lem:v-value-set}
At the end of stage $k \in [K]$ and epoch $\tau \in [n]$, for each level $\ell \in [L]$, group $g \in [G]$, we have
\begin{align*}
&~ V^{*}_{t(k, \tau), h(\ell, g)-1}(s^{\mb}_{h(\ell, g)-1, i})\\
=&~ \max_{j \in [A/2]}\left\{ \frac{|C_\tau \cap B_{k, N(g, i, j)}|}{b} \cdot \frac{H}{2} + \left(1- \frac{|C_\tau \cap B_{k, N(g, i, j)}|}{b}\right) \cdot V^{*}_{t(k, \tau), h(\ell)+1}(s^{\mmp}_{h(\ell)+1})\right\}
\end{align*}
\end{lemma}

\begin{lemma}[$V$-value, pivotal state]
\label{lem:v-value-pivotal}
At the end of stage $k \in [K]$ and epoch $\tau \in [n]$, for each level $\ell \in [L]$, the $V$-value of the pivotal state satisfies
\begin{align*}
&~ V^{*}_{t(k, \tau), h(\ell-1)+1}(s^{\mmp}_{h(\ell-1)+1}) \\
= &~ \max_{\nu\in [N]} \left\{\frac{|C_\tau \cap B_{k,\nu}|}{b} \cdot \frac{H}{2} + \left(1- \frac{|C_\tau \cap B_{k, \nu}|}{b}\right) \cdot V^{*}_{t(k, \tau), h(\ell)+1}(s^{\mmp}_{h(\ell)+1})\right\}.
\end{align*}
\end{lemma}

As a corollary, we can compute the $V$-value of the initial state. 
\begin{lemma}[$V$-value, initial state]
\label{lem:value-init}
Let $\kappa_{k, \tau} = \max_{\nu\in [N]}\frac{|C_\tau\cap B_{k, \nu}|}{b}$, then at the end of stage $k$ and epoch $\tau \in [n]$, one has
\[
V_{t(k, \tau), 1}^{*}(s_{\init}) = (1 - (1-\kappa_{k, \tau})^{L}) \cdot \frac{H}{2}.
\]
\end{lemma}

Now we can complete the proof of Theorem \ref{thm:fully}
\begin{proof}[Proof of Theorem \ref{thm:fully}]
If the input of $\mip$ is a YES instance, suppose $C_\tau \subseteq B_{k, \nu}$ for some $\tau \in [n], k\in [K], \nu \in [N]$; then $\kappa_{k, \tau} = c/b = 1/L$. By Lemma \ref{lem:value-init}, the value of $s_{\init}$ at the end of epoch $t$ satisfies
\[
V_{t(k, \tau),1}^{*}(s_{\init}) = (1- (1- \kappa_{k, \tau})^{L}) \cdot \frac{H}{2} = (1-(1-1/L)^{L})\cdot \frac{H}{2} \geq \frac{H}{4}.
\]

In the NO instance case, we have 
\[
\kappa_{k, \tau} \leq c/wb \quad \text{where} \quad w = 2^{\log(n)^{1-o(1)}} = \Omega(1),
\]
then the value of $s_{\init}$ at the end of any stage $k \in [K]$, epoch $\tau\in [n]$ is at most
\[
V_{t(k, \tau),1}^{*}(s_{\init}) = (1- (1- \kappa_{k, \tau})^{L}) \cdot H/2 \leq (1-(1-1/wL)^{L})\cdot \frac{H}{2} \leq \frac{1}{w}\cdot \frac{H}{2}  \leq \frac{H}{100}.
\]

Now we bound the amortized runtime. By Theorem \ref{thm:arw}, assuming $\SETH$, the total runtime of any NSRL algorithm should be at least $n^{2-o(1)}$, and therefore, the amortized runtime per update should be at least $n^{2-o(1)}/T = (SAH)^{1-o(1)}\cdot T^{-o(1)} \approx (SAH)^{1-o(1)}$ when $T = \poly(SAH)$.
This completes the proof.
\end{proof}

\begin{remark}
The statement of Theorem \ref{thm:fully} asserts the decision version of NSRL requires $(SAH)^{1-o(1)}$ time per update. The same lower bound translates directly to the task of maintaining an approximate $V$-value or maintaining an approximately optimal policy. 
\end{remark}

\section{Incremental action changes}
\label{sec:incremental}

When the MDP changes only through the introduction of new actions, then we can maintain an $\eps$-approximation to value with amortized runtime that depends, polynomially, only on $H$ and $\tfrac{1}{\eps}$ (and not $S$). 
\begin{theorem}[Efficient algorithm, incremental changes]
\label{thm:incremental-algo}
\label{thm:incremental}
There is an algorithm with amortized runtime $\wt{O}(H^5/\eps^3)$ per update that maintains an $\eps$-approximation of the value over any sequence of $T$ insertions of actions.
\end{theorem}

The approach is given as Algorithm \ref{algo:lazy-qvi}.
It combines the classic $Q$-value iteration with lazy updates on $V$-value. For each new state-action pair $(s_h, a_h)$, it constructs the empirical transition kernel using samples from $P_h(s_h, a_h)$. The newly added action could potentially affect the state value, and our algorithm propagates the change --- lazily --- to downstream states.  That is, a change to $V$-value is triggered only if it significantly exceeds the previous estimate.
The key mathematical intuition is the monotonicity of $V$-value under incremental action changes.
The amortized runtime of Algorithm \ref{algo:lazy-qvi} is bounded because the $Q$-value of each state-action is updated rarely, at most $\wt{O}(H^3/\eps^2 \cdot H^2/\eps) = \wt{O}(H^5/\eps^3)$ times, due to the sparsity of the empirical transition kernel and the lazy updates. The correctness of our algorithm follows from the standard Bernstein type bound and a robust analysis of $Q$-value iteration. The detailed proof can be found at Appendix \ref{sec:incremental-app}.

\begin{algorithm}[!htbp]
\caption{Lazy updated $Q$-value iteration (Lazy-QVI)}
\label{algo:lazy-qvi}
\begin{algorithmic}[1]
\State Initialize $N \leftarrow  H^3\log^3(SHT)/\eps^2$, $\wh{V}_h(s_h) \leftarrow 0, \wt{V}_h(s_h) \leftarrow 0$, $\forall s_h \in \mS_h , h\in [H]$
\Procedure{{\sc Insert}}{$s_h, a_h$} 
\State Generate $N$ samples $\{\wh{s}_{h+1,1}, \ldots, \wh{s}_{h+1, N}\}$  from ${P}_{h}(s_h, a_h)$ and reward $r_h(s_h, a_h)$ 
\State $\wh{P}_{h}(s_h, a_h) \leftarrow \unif\{\wh{s}_{h+1,1}, \ldots, \wh{s}_{h+1, N}\}$
\State Call {\sc Propagate}
\EndProcedure
\Procedure{{\sc Propagate}}{}
\For{$h = H, H-1, \ldots, 1$}
\For{state-action pair $(s_h, a_h) \in \mS_h \times \mA_h$} \Comment{Update only if there is a change}
\State $\wh{Q}_h(s_h, a_h) \leftarrow r_{h}(s_h, a_h) + \E_{s_{h+1}\sim \wh{P}_{h}(s_h, a_h)} \wt{V}_{h+1}(s_{h+1})$ 
\State $\wh{V}_h(s_h) \leftarrow \max_{a_h} \wh{Q}(s_h, a_h)$
\State \textbf{If} $\wt{V}_h(s_h) \leq \wh{V}_h(s_h) - \eps/4H$ \textbf{then} $\wt{V}_h(s_h) \leftarrow \wh{V}_h(s_h)$
\EndFor
\EndFor
\EndProcedure
\end{algorithmic}
\end{algorithm}

Theorem \ref{thm:incremental} provides an efficient algorithm for approximately optimal policy, one natural question is whether one can maintain the exact optimal policy (or value function) under incremental action changes. We give a negative answer, showing that $T^{1-o(1)}$ runtime is necessary if one wants to maintain an $O(1/T)$-approximation to the value of optimal policy.

\begin{theorem}[Lower bound, exact optimal policy]
\label{thm:incremental-lower}
Unless $\SETH$ is false, there is a sequence of $T$ action insertions such that no algorithm with amortized runtime $T^{1-o(1)}$ per update can maintain an $O(1/T)$-approximation to the value of optimal policy.
\end{theorem}

\section{Discussion}
\label{sec:discussion}
The importance of a complexity result rests on its capacity to inform the development of new algorithms. Our result seems to suggest that a successful heuristic approach to NSRL can alternate between additional exploration after each change in parameters and, when this brings diminishing benefits, a restart from scratch.  This is not unlike some of the approaches taken by some state-of-the-art applications \cite{padakandla2021survey}.  By further developing this and similar approaches, the current challenge of NSRL may be eventually tamed.  We also note that our negative result leaves open the NSRL problem in the case of function approximation \cite{jin2020provably,agarwal2019reinforcement}; we conjecture that a similar negative result may be provable in this case as well.

\bibliographystyle{alpha}
\bibliography{ref}

\clearpage
\newpage
\appendix
\section{Additional related work}
\label{sec:relate-app}
\paragraph{Computational complexity of reinforcement learning}
The computational complexity of (stationary) MDP has been a central topic across multiple disciplines. The study of MDP dates back to Bellman \cite{bellman1957dynamic} in 1950s, and since then, there is a long line of work concerning the computational efficiency of MDP \cite{tseng1990solving,littman1995complexity,howard1960dynamic,ye2011simplex,scherrer2013improved,ye2005new,sidford2018variance,sidford2018near,sidford2020solving,lee2014path,van2021minimum,papadimitriou1987complexity}.
The classical approaches include value iteration, policy iteration and linear programming, see \cite{puterman2014markov,bertsekas2012dynamic} for reference.
For a finite horizontal MDP with $S$ states, $A$ actions and $H$ steps, the value iteration could return the optimal policy with runtime $O(S^2AH)$ that is linear in the input size. If the algorithm could sample from the transition function (a.k.a. the generative model), then 
\cite{gheshlaghi2013minimax} provide an algorithm that returns an $\eps$-approximation to the $V$-value with runtime $\wt{O}(SAH^3/\eps^2)$. 
For non-stationary MDP, it implies an algorithm with runtime $\wt{O}(S^2AH + SAH^3T/\eps^2)$ for $\eps$-value approximation over a sequence of $T$ updates. This is because the algorithm could always re-compute from scratch, and it can sample the transition function in $O(\log(S))$ time using a binary tree data structure, after reading the input initially.

Besides computation complexity, a large number of work concern about the sample complexity in generative model (e.g. \cite{gheshlaghi2013minimax,li2020breaking}) and regret in model-free RL (e.g. \cite{jin2018q}), in tabular setting as well as functional approximation setting \cite{foster2021statistical}.

\section{Missing proof from Section \ref{sec:fully}}
\label{sec:fully-app}

\begin{proof}[Proof of Lemma \ref{lem:v-value-terminal}]
This is quite obvious, as the terminal state always stays on itself (Eq.~\eqref{eq:terminal-transition}), the reward of $s_{h,2}^{\mt}$ is always $0$, while the reward of $s_{h,1}^{\mt}$ is $0$ in phase one and $1$ in phase two (Eq.~\eqref{eq:reward}).
\end{proof}

\begin{proof}[Proof of Lemma \ref{lem:v-value-element}]
For an element $u \in C_\tau$, a policy could choose to never leave $s^{\me}_{u}$ (Eq.~\eqref{eq:element-transition1}), and it receives the maximum $H/2$ reward (see Eq.~\eqref{eq:transition-new}\eqref{eq:reward}). 
For an element $u \notin C_\tau$, the policy needs to stay at $s^{\me}_u$ until the end of layer $\ell$ (see Eq.~\eqref{eq:element-transition1}).
While at the end of layer $\ell$, it could move to pivotal state $s^{\mmp}_{h(\ell)+1}$ or stay on itself (Eq.~\eqref{eq:element-transition2}). The later obtains strictly less reward, because the pivotal state could always stay on itself (Eq.~\eqref{eq:pivotal-transition}), and the value $V^{*}_{t(k, \tau), H/2}(s^{\me}_{H/2,u}) = 0$ at the end of phase $1$ (see Eq.~\eqref{eq:transition-new}\eqref{eq:reward}). This completes the proof.
\end{proof}

\begin{proof}[Proof of Lemma \ref{lem:v-value-set}]
The $Q$-value of choosing action $a_{j}^{\mb}$ ($j\in [A/2]$) equals 
\begin{align*}
&~ Q^{*}_{t(k, \tau), h(\ell, g)-1}(s^{\mb}_{h(\ell, g)-1, i}, a_{j}^{\mb})\\
= &~  \sum_{u\in [m]}\Pr[s_{h(\ell, g)} = s^{\me}_{h(\ell, g), u}] \cdot V^{*}_{t(k,\tau), h(\ell, g)}(s^{\mb}_{h(\ell, g), u}) \\
= &~ \sum_{u\in C_\tau}\Pr[s_{h(\ell, g)} = s^{\me}_{h(\ell, g), u}] \cdot \frac{H}{2} + \sum_{u\in [m]\backslash C_\tau}\Pr[s_{h(\ell, g)} = s^{\me}_{h(\ell, g), u}]\cdot  V_{t(k,\tau), h(\ell)+1}^{*}(s^{\mmp}_{h(\ell)+1})\\
= &~ \frac{|C_\tau \cap B_{k,N(g, i, j)}|}{b} \cdot \frac{H}{2} + \left(1- \frac{|C_\tau \cap B_{k,N(g, i, j)}|}{b}\right) \cdot V^{*}_{t(k, \tau), \ell+1}(s^{\mmp}_{h(\ell)+1}).
\end{align*}
The first step follows from Bellman's equation and the state-action pair $(s^{\mb}_{h(\ell, g)-1, i}, a_{j}^{\mb})$ receives $0$ reward (Eq.~\eqref{eq:reward}), the second step follows from Lemma \ref{lem:v-value-element}, the last step follows from Eq.~\eqref{eq:state-transition}. The proof follows by taking the maximum over action $\{a_j^{\mb}\}_{j\in [A/2]}$.
\end{proof}

\begin{proof}[Proof of Lemma \ref{lem:v-value-pivotal}]
For each level $\ell \in [L]$, the transition functions of the pivotal state and routing states guarantee that a policy can visit exactly one set state in $\mH_\ell$. 
To see this, it can visit at most one set state because the element state stays on itself till the end of the layer (Eq.~\eqref{eq:element-transition1}).  Meanwhile, it can go to any set state $\nu \in [N]$ with $\nu = N(g, i, j)$ for some $i \in [S/2]$ and $g\in [G]$, because it can first go to the pivotal state $s^{\mmp}_{h(\ell, g-1) + 1,0}$ at the beginning of group $g$, then move to $s^{\mr}_{h(\ell, g) -1, i}$ through routing states (see Eq.~\eqref{eq:routing1}\eqref{eq:routing2}). Combining Lemma \ref{lem:v-value-set}, we have
\begin{align*}
&~ V^{*}_{t(k, \tau), h(\ell-1)+1}(s^{\mmp}_{h(\ell-1)+1})\\ = &~ \max_{g\in [G]}\max_{i \in [S/4]}V^{*}_{t(k, \tau), h(\ell, g)-1}(s^{\mb}_{h(\ell, g)-1, i})\\
= &~\max_{g\in [G]}\max_{i\in [S/4]}\max_{j \in [A/2]}\left\{ \frac{|C_\tau \cap B_{k, N(g, i, j)}|}{b} \cdot \frac{H}{2} + \left(1- \frac{|C_\tau \cap B_{k, N(g, i, j)}|}{b}\right) \cdot V^{*}_{t(k, \tau), h(\ell)+1}(s^{\mmp}_{h(\ell)+1})\right\}\\
= &~ \max_{\nu\in [N]} \left\{\frac{|C_\tau \cap B_{k, \nu}|}{b} \cdot \frac{H}{2} + \left(1- \frac{|C_\tau \cap B_{k, \nu}|}{b}\right) \cdot V^{*}_{t(k, \tau), h(\ell)+1}(s^{\mmp}_{h(\ell)+1})\right\}.
\end{align*}
This completes the proof of the lemma.
\end{proof}

\begin{proof}[Proof of Lemma \ref{lem:value-init}]
By Lemma \ref{lem:v-value-pivotal}, for any $\ell \in [L]$, we have 
\begin{align*}
&~ V^{*}_{t(k, \tau), h(\ell-1)+1}(s^{\mmp}_{h(\ell-1)+1})\\
= &~ \max_{\nu\in [N]} \left\{\frac{|C_\tau \cap B_{k, \nu}|}{b} \cdot \frac{H}{2} + \left(1- \frac{|C_\tau \cap B_{k, \nu}|}{b}\right) \cdot V^{*}_{t(k, \tau), h(\ell)+1}(s^{\mmp}_{h(\ell)+1})\right\}\\
= &~ \kappa_{k, \tau} \cdot \frac{H}{2} + (1-\kappa_{k, \tau}) V^{*}_{t(k, \tau), h(\ell)+1}(s^{\mmp}_{h(\ell)+1}).
\end{align*}
Solving the above recursion, one has
\begin{align*}
V_{t(k, \tau), 1}^{*}(s_{\init}) = V^{*}_{t(k, \tau), 1}(s^{\mmp}_{1}) = \sum_{\ell=1}^{L} \kappa_{k, \tau} (1 - \kappa_{k, \tau})^{\ell-1} \cdot \frac{H}{2} =(1 - (1-\kappa_{k, \tau})^{L})\cdot \frac{H}{2}.
\end{align*}
This completes the proof of the lemma.
\end{proof}

\section{Missing proof from Section \ref{sec:incremental}}
\label{sec:incremental-app}

We first state the concentration bounds used in the paper.
\begin{lemma}[Hoeffding bound]\label{lem:hoeffding}
Let $X_1, \cdots, X_n$ be $n$ independent bounded variables in $[a_i,b_i]$. Let $X= \sum_{i=1}^n X_i$, then we have
\begin{align*}
\Pr[ | X - \E[X] | \geq t ] \leq 2\exp \left( - \frac{2t^2}{ \sum_{i=1}^n (b_i - a_i)^2 } \right).
\end{align*}
\end{lemma}

\begin{lemma}[Bernstein bound]\label{lem:bernstein}
Let $X_1, \cdots, X_n$ be $n$ independent zero mean random variables and $|X_i| \leq M$. 
Let $X= \sum_{i=1}^n X_i$, $\sigma = \sum_{i=1}^n \E[X_i^2]$ then we have
\begin{align*}
\Pr[ | X | \geq t ] \leq 2\exp \left( - \frac{2t^2}{Mt/3 + \sigma^2} \right).
\end{align*}
In particular, with probability at least $1-\delta$, one has
\begin{align*}
 | X | \leq (M/3 + \sigma)\cdot\log(1/\delta).
\end{align*}
\end{lemma}

We prove Algorithm \ref{algo:lazy-qvi} gives $\eps$-approximation to both $V$-value and $Q$-value. For notation convenience, we drop the subscript of round number $t$ in the proof.
\begin{lemma}[Value approximation]
\label{lem:incremental-value-approx}
At the end of $t$-th update ($t\in [T]$), for any step $h \in [H]$, state $s_h \in \mS_h$ and action $a_h$, with probability at least $1-(SHT)^{-\omega(1)}$, we have
\[
|V_h^{*}(s_h) - \wh{V}_h(s_h)| \leq \eps/2
\quad \text{and} \quad
|Q_h^{*}(s_h, a_h) - \wh{Q}_h(s_h, a_h)| \leq \eps
\]
\end{lemma}

\begin{proof}
We prove the claim by induction. 
The base case of $h = H$ holds trivially, as there is no error.
Suppose the claim holds up to step $h+1$, then for the $h$-th step, we have
\begin{align}
\wh{Q}_{h}(s_h, a_h) = &~ r_h(s_h, a_h) + \E_{s_{h+1}\sim \wh{P}_{h}(s_h, a_h)} \wt{V}_{h+1}(s_{h+1}) \notag \\
= &~ r_h(s_h, a_h) + \E_{s_{h+1}\sim \wh{P}_{h}(s_h, a_h)} \wh{V}_{h+1}(s_{h+1}) \pm \frac{\eps}{4H} \notag \\
= &~ r_h(s_h, a_h) + \E_{s_{h+1}\sim \wh{P}_{h}(s_h, a_h)} [V^{*}_{h+1}(s_{h+1})] \notag \\
&~ + \E_{s_{h+1}\sim \wh{P}_{h}(s_h, a_h)} [\wh{V}_{h+1}(s_{h+1}) - V^{*}_{h+1}(s_{h+1})]  \pm \frac{\eps}{4H}, \label{eq:incremental-value1}
\end{align}
where the first step follows from the update rule of Algorithm \ref{algo:lazy-qvi}, the second step holds since that the propagate value $\wt{V}_{h+1}(s_{h+1})$ satisfies
\[
\left|\wh{V}_{h+1}(s_{h+1}) - \wt{V}_{h+1}(s_{h+1})\right| \leq \frac{\eps}{4H}, \quad \forall s_{h+1} \in \mS_{h+1}.
\]

We bound the second term of Eq.~\eqref{eq:incremental-value1} in terms of variance.
Define
\begin{align*}
\sigma_h(s_h, a_h)^2 := \E_{s_{h+1}\sim P_{h}(s_h, a_h)}[V_{h+1}^{*}(s_{h+1})^2] -  \left(\E_{s_{h+1}\sim P_{h}(s_h, a_h)}[V_{h+1}^{*}(s_{h+1})]\right)^2.
\end{align*}
By Bernstein inequality, we have with probability at least $1 - (SHT)^{-\omega(1)}$,
\begin{align*}
\left|\E_{s_{h+1}\sim \wh{P}_{h}(s_h, a_h)} [V^{*}_{h+1}(s_{h+1})] - \E_{s_{h+1}\sim P_{h}(s_h, a_h)} [V^{*}_{h+1}(s_{h+1})]\right|
\lesssim &~ \frac{H + \sqrt{N}\sigma_h(s_h, a_h)}{N} \cdot \log(SHT) \\
\leq &~ \frac{\eps^2}{H^2} + \frac{\eps}{16H^{3/2}}\cdot\sigma_{h}(s_h, a_h).
\end{align*}

Plugging into Eq.~\eqref{eq:incremental-value1}, we have
\begin{align}
&~ \wh{Q}_{h}(s_h, a_h) \notag \\
= &~ r_h(s_h, a_h) + \E_{s_{h+1}\sim P_{h}(s_h, a_h)}[V^{*}_{h+1}(s_{h+1})] + \E_{s_{h+1}\sim \wh{P}_{h}(s_h, a_h)} [\wh{V}_{h+1}(s_{h+1}) - V^{*}_{h+1}(s_{h+1})] \notag \\
&~ \pm \frac{\eps}{16H^{3/2}}\cdot \sigma_h(s_h, a_h) \pm \frac{\eps}{3H} \notag \\
= &~ Q^{*}_h(s_h, a_h) + \E_{s_{h+1}\sim \wh{P}_{h}(s_h, a_h)} [\wh{V}_{h+1}(s_{h+1}) - V^{*}_{h+1}(s_{h+1})]\pm \frac{\eps}{16H^{3/2}}\cdot \sigma_h(s_h, a_h) \pm \frac{\eps}{3H}.\label{eq:incremental-value2}
\end{align}

We bound the $V$-value difference $\wh{V}_{h}(s_{h}) - V^{*}_{h}(s_{h})$ and provide upper and lower bounds separately.

{\bf Upper bound $\wh{V}_{h}(s_{h}) - V^{*}_{h}(s_{h})$.}
Let $\wh{\pi}$ be the policy induced by $\wh{Q}$, that is, for any state $s_\ell \in \mS_\ell$, $\wh{\pi}(s_\ell) = \mathsf{argmax}_{a_\ell}\wh{Q}_\ell(s_\ell, a_\ell)$.
Then for any state $s_h \in \mS_h$, one has
\begin{align}
&~ \wh{V}_{h}(s_{h}) - V_{h}^{*}(s_{h})\notag \\
= &~ \wh{Q}_{h}(s_{h}, \wh{\pi}(s_h)) - Q^{*}_{h}(s_{h}, \pi^{*}(s_h)) \notag \\
= &~  \wh{Q}_{h}(s_{h}, \wh{\pi}(s_h)) - Q^{*}_{h}(s_{h}, \wh{\pi}(s_h)) + Q^{*}_{h}(s_{h}, \wh{\pi}(s_h)) - Q^{*}_{h}(s_{h}, \pi^{*}(s_h))\notag \\
\leq &~ \wh{Q}_{h}(s_{h}, \wh{\pi}(s_h)) - Q^{*}_{h}(s_{h}, \wh{\pi}(s_h))\notag \\
\leq &~ \E_{s_{h+1}\sim \wh{P}_{h}(s_h, \wh{\pi}(s_h))} [\wh{V}_{h+1}(s_{h+1}) - V^{*}_{h+1}(s_{h+1})] \pm \frac{\eps}{16H^{3/2}}\sigma_h(s_h, \wh{\pi}(s_h)) + \frac{\eps}{3H}, \label{eq:incremental-value3}
\end{align}
where the third step follows from the optimality of $\pi^{*}$, the fourth step follows from Eq.~\eqref{eq:incremental-value2}.

Fix the state $s_h \in \mS_{h}$, for any step $\ell \in [h: H]$ and state $s_{\ell} \in \mS_{\ell}$, let $\wh{p}(s_{\ell})$ be the probability that policy $\hat{\pi}$ goes to state $s_{\ell}$, starting from $s_h$.
Recurring Eq.~\eqref{eq:incremental-value3}, we obtain
\begin{align}
\wh{V}_{h}(s_{h}) - V^{*}_{h}(s_{h}) \leq &~ \frac{\eps}{16H^{3/2}} \cdot \sum_{\ell = h}^{H}\sum_{s_\ell \in \mS_\ell} \wh{p}(s_\ell)\sigma_{\ell}(s_\ell, \wh{\pi}(s_\ell)) + \frac{\eps}{3} \notag \\
\leq &~ \frac{\eps}{16H} \sqrt{\sum_{\ell=h}^{H} \wh{p}(s_\ell) \sigma_{\ell}(s_\ell, \wh{\pi}(s_\ell))^2} + \frac{\eps}{3}. \label{eq:incremental-value4}
\end{align}
Here the first step follows Eq.~\eqref{eq:incremental-value3}, the second step follows from Cauchy Schwarz inequality and $\sum_{s_\ell\in \mS_\ell}\wh{p}(s_\ell) = 1$ holds for any $\ell \geq h$.

We need the following two technical Lemmas.
\begin{lemma}[Connection with empirical variance]
\label{lem:variance-connection}
Define the empirical variance
\begin{align*}
\wh{\sigma}_{h}(s_h, a_h)^{2} := \E_{s_{h+1}\sim \wh{P}_{h}(s_h, a_h)}[\wh{V}_{h+1}(s_{h+1})^2] -  \left(\E_{s_{h+1}\sim \wh{P}_{h}(s_h, a_h)}[\wh{V}_{h+1}(s_{h+1})]\right)^2
\end{align*}
Then with probability at least $1-(SHT)^{-\omega(1)}$, one has
\begin{align*}
|\sigma_h(s_h, a_h)^2 - \wh{\sigma}_h(s_h, a_h)^2| \leq  H.
\end{align*}
\end{lemma}
\begin{proof}
First, by Hoeffding inequality, with probability at least $1-(SHT)^{-\omega(1)}$, one has
\begin{align*}
\left|\E_{s_{h+1}\sim \wh{P}_{h}(s_h, a_h)}[\wh{V}_{h+1}(s_{h+1})^2] - \E_{s_{h+1}\sim P_{h}(s_h, a_h)}[\wh{V}_{h+1}(s_{h+1})^2]\right| \leq \frac{4H^2\sqrt{N}\log(SHT)}{N} \leq  H/4.
\end{align*}
By induction hypothesis, one has $|\wh{V}_{h+1}(s_{h+1}) - V_{h+1}(s_{h+1})| \leq \eps$ for any state $s_{h+1} \in \mS_{h+1}$, and therefore,
\[
\left|\E_{s_{h+1}\sim P_{h}(s_h, a_h)}[\wh{V}_{h+1}(s_{h+1})^2 - V_{h+1}(s_{h+1})^2] \right| \leq 2\eps H \leq H/4
\]
Similarly, by Hoeffding bound, we have with probability at least $1-(SHT)^{-\omega(1)}$, 
\begin{align*}
\left|\E_{s_{h+1}\sim \wh{P}_{h}(s_h, a_h)}[\wh{V}_{h+1}(s_{h+1})] - \E_{s_{h+1}\sim P_{h}(s_h, a_h)}[\wh{V}_{h+1}(s_{h+1})]\right| \leq \frac{4\sqrt{N}H\log(SHT)}{N} \leq \eps
\end{align*}
and by induction hypothesis,
\[
\left|\E_{s_{h+1}\sim P_{h}(s_h, a_h)}[\wh{V}_{h+1}(s_{h+1}) - V_{h+1}(s_{h+1})] \right| \leq \eps 
\]
Combining the above four inequalities, one can conclude the proof.
\end{proof}

\begin{lemma}[Upper bound on empirical variance]
\label{lem:upper-empirical}
We have
\[
\sum_{\ell=h}^{H} \wh{p}(s_\ell) \wh{\sigma}_{\ell}(s_\ell, \wh{\pi}(s_\ell))^2 \leq 3H^2
\]
\end{lemma}
\begin{proof} We have
\begin{align*}
&~ \sum_{\ell=h}^{H}\sum_{s_\ell \in \mS_\ell} \wh{p}(s_\ell) \wh{\sigma}_{\ell}(s_\ell, \wh{\pi}(s_\ell))^2 \\
= &~ \sum_{\ell=h}^{H}\sum_{s_\ell \in \mS_\ell}\wh{p}(s_\ell)  \cdot \left(\E_{s_{\ell+1}\sim \wh{P}_{\ell}(s_\ell, \wh{\pi}(s_\ell))}[\wh{V}_{\ell+1}(s_{\ell+1})^2] -  \left(\E_{s_{\ell+1}\sim \wh{P}_{\ell}(s_\ell, \wh{\pi}(s_\ell))}[\wh{V}_{\ell+1}(s_{\ell+1})]\right)^2\right)\\
\leq &~ \sum_{\ell=h+1}^{H} \sum_{s_\ell \in \mS_\ell} \wh{p}(s_{\ell})\left(\wh{V}_{\ell}(s_{\ell})^2 - 
\left(\E_{s_{\ell+1}\sim \wh{P}(s_\ell, \wh{\pi}(s_\ell))}[\wh{V}_{\ell+1}(s_{\ell+1})]\right)^2  \right)  + 1\\
= &~ \sum_{\ell=h+1}^{H} \sum_{s_\ell \in \mS_\ell} \wh{p}(s_{\ell})\left(\left(\E_{s_{\ell+1}\sim \wh{P}(s_\ell, \wh{\pi}(s_\ell))}[\wt{V}_{\ell+1}(s_{\ell+1}) + r_\ell(s_\ell, \wh{\pi}(s_{\ell}))]\right)^2 - 
\left(\E_{s_{\ell+1}\sim \wh{P}(s_\ell, \wh{\pi}(s_\ell))}[\wh{V}_{\ell+1}(s_{\ell+1})]\right)^2  \right) + 1\\
\leq &~ \sum_{\ell=h+1}^{H}\sum_{s_\ell \in \mS_\ell}\wh{p}(s_\ell) \cdot 2H \cdot (1+\eps/H) + 1 \leq 3H^2.
\end{align*}
The first step follows from the definition of empirical variance $\wh{\sigma}_\ell$. The second step is important and it holds due to the definition of visiting probability $\wh{p}_\ell$, and we use the naive bound of  $\wh{V}_{H}(s_{H}) \leq 1$ for any state $s_{H} \in \mS_H$ in last step. The third step holds due to the definition of $\wh{V}_\ell(s_\ell)$. The last step holds due to
\[
\left|\E_{s_{\ell+1}\sim \wh{P}(s_\ell, \wh{\pi}(s_\ell))} \wt{V}_{\ell+1}(s_{\ell+1}) + r_\ell(s_\ell, \wh{\pi}(s_{\ell})) - \wh{V}_{\ell+1}(s_{\ell+1})\right| \leq 1 + \eps/H
\]
as $|\wt{V}_{\ell+1}(s_{\ell+1}) - \wh{V}_{\ell+1}(s_{\ell+1})| \leq \eps/H$ and $r_\ell(s_\ell, \wh{\pi}(s_{\ell})) \leq 1$.
\end{proof}

Combining Lemma \ref{lem:variance-connection}, Lemma \ref{lem:upper-empirical} and Eq.~\eqref{eq:incremental-value4}, we have that
\begin{align}
\wh{V}_{h}(s_{h}) - V^{*}_{h}(s_{h})  \leq &~ \frac{\eps}{16H} \sqrt{\sum_{\ell=h}^{H} \wh{p}(s_\ell) \sigma_{\ell}(s_\ell, \wh{\pi}(s_\ell))^2} + \frac{\eps}{3} \notag \\
=&~ \frac{\eps}{16H} \sqrt{\sum_{\ell=h}^{H} \wh{p}(s_\ell)\wh{\sigma}_{\ell}(s_\ell, \wh{\pi}(s_\ell))^2 + H^2} + \frac{\eps}{3} \notag \\
\leq &~  \frac{\eps}{16H}\cdot \sqrt{3H^2 + H^2} + \frac{\eps}{3} \leq \frac{\eps}{2} \label{eq:incremental-v-upper}
\end{align}

{\bf Lower bound $\wh{V}_{h}(s_{h}) - V^{*}_{h}(s_{h})$.} The proof for lower bound is similar. First, we have
\begin{align}
&~ V^{*}_{h}(s_{h}) - \wh{V}_{h}(s_{h})\notag\\
= &~ Q^{*}_{h}(s_{h}, \pi^{*}(s_h)) - \wh{Q}_{h}(s_{h}, \wh{\pi}(s_h)) \notag \\
= &~ Q^{*}_{h}(s_{h}, \pi^{*}(s_h)) - \wh{Q}_h(s_{h}, \pi^{*}(s_h))  + \wh{Q}_h(s_{h}, \pi^{*}(s_h)) - \wh{Q}_{h}(s_{h}, \wh{\pi}(s_h))\notag \\
\leq &~ Q^{*}_{h}(s_{h}, \pi^{*}(s_h)) - \wh{Q}_h(s_{h}, \pi^{*}(s_h))\notag \\
\leq &~ \E_{s_{h+1}\sim \wh{P}_{h}(s_h, \pi^{*}(s_h))} [V_{h+1}^{*}(s_{h+1}) - \wh{V}_{h+1}(s_{h+1})] + \frac{\eps}{16H^{3/2}}\sigma_h(s_h, \pi^{*}(s_h)) + \frac{\eps}{3H}, \label{eq:incremental-value5}
\end{align}
where the third step follows from the optimality of $\wh{\pi}$, the fourth step follows from Eq.~\eqref{eq:incremental-value2}.

Using Hoeffding bound and the induction hypothesis $|V_{h+1}^{*}(s_{h+1}) - \wh{V}_{h+1}(s_{h+1})| \leq \eps$, with probability at least $1 - (SHT)^{-\omega(1)}$, we have
\begin{align*}
&~ \E_{s_{h+1}\sim \wh{P}_{h}(s_h, \pi^{*}(s_h))} [V_{h+1}^{*}(s_{h+1}) - \wh{V}_{h+1}(s_{h+1})] - \E_{s_{h+1}\sim P_{h}(s_h, \pi^*(s_h))} [V_{h+1}^{*}(s_{h+1}) - \wh{V}_{h+1}(s_{h+1})]  \\
\leq &~ \frac{2\eps\sqrt{N}\log(SHT)}{N} \leq \frac{\eps}{24H}. 
\end{align*}

Plug into Eq.~\eqref{eq:incremental-value5}, we have
\begin{align*}
V^{*}_{h}(s_{h}) - \wh{V}_{h}(s_{h}) \leq &~ \E_{s_{h+1}\sim P_{h}(s_h, \pi^*(s_h))} [V_{h+1}^{*}(s_{h+1}) - \wh{V}_{h+1}(s_{h+1})]\\
&~+ \frac{\eps}{16H^{3/2}}\sigma_h(s_h, \pi^{*}(s_h)) + \frac{3\eps}{8H}.
\end{align*}

Fix the state $s_h \in \mS_{h}$, for any step $\ell \in [h: H]$ and state $s_{\ell} \in \mS_{\ell}$, let $p^{*}(s_{\ell})$ be the probability that policy $\pi^{*}$ goes to state $s_{\ell}$, starting from $s_h$.
Recurring the above equation, we obtain
\begin{align}
V^{*}_{h}(s_{h}) - \wh{V}_{h}(s_{h})  \leq &~ \frac{\eps}{16H^{3/2}} \cdot \sum_{\ell = h}^{H}\sum_{s_\ell \in \mS_\ell} p^{*}(s_\ell)\sigma_{\ell}(s_\ell, \pi^{*}(s_\ell)) + \frac{3\eps}{8} \notag \\
\leq &~ \frac{\eps}{16H} \sqrt{\sum_{\ell=h}^{H} p^{*}(s_\ell) \sigma_{\ell}(s_\ell, \pi^{*}(s_\ell))^2} + \frac{3\eps}{8}. \notag \\
\leq &~ \frac{\eps}{16H} \sqrt{3H^2} + \frac{\eps}{3}\leq \eps/2. \label{eq:incremental-v-lower}
\end{align}
We use Cauchy Schwarz in the second step, the third step follows from the following Lemma (the proof is similar to Lemma \ref{lem:upper-empirical} and we omit it here).
\begin{lemma}[Upper bound on variance]
\label{lem:upper-variance}
We have
\[
\sum_{\ell=h}^{H} p^{*}(s_\ell) \sigma_{\ell}(s_\ell, \pi^{*}(s_\ell))^2 \leq 3H^2.
\]
\end{lemma}

Combining Eq.~\eqref{eq:incremental-v-upper} and Eq.~\eqref{eq:incremental-v-lower}, we conclude the proof for $V$-value. 

For $Q$-value, we have with probability at least $1- (SHT)^{-\omega(1)}$,
\begin{align*}
\wh{Q}_{h}(s_h, a_h) = &~ r_h(s_h, a_h) + \E_{s_{h+1}\sim \wh{P}_{h}(s_h, a_h)} \wt{V}_{h+1}(s_{h+1}) \notag \\
= &~ r_h(s_h, a_h) + \E_{s_{h+1}\sim \wh{P}_{h}(s_h, a_h)} \wh{V}_{h+1}(s_{h+1}) \pm \frac{\eps}{4H} \notag \\
= &~ r_h(s_h, a_h) + \E_{s_{h+1}\sim \wh{P}_{h}(s_h, a_h)} V^{*}_{h+1}(s_{h+1}) \pm \eps/2 \pm \frac{\eps}{4H}  \notag \\
= &~ r_h(s_h, a_h) + \E_{s_{h+1}\sim P_{h}(s_h, a_h)} V^{*}_{h+1}(s_{h+1}) \pm \eps\\
= &~ Q^*_{h}(s_h, a_h)\pm \eps.
\end{align*}
The first step uses the update rule of Algorithm \ref{algo:lazy-qvi}, the second step holds since $|\wt{V}_{h+1}(s_{h+1}) - \wh{V}_{h+1}(s_{h+1})| \leq \eps/4H$. The third follows from the guarantee of $V$-value, and the fourth step follows from Hoeffding bounds and the last step follows from Bellman equation. We finish the induction and complete the proof here.
\end{proof}

We next bound the total update time of Algorithm \ref{algo:lazy-qvi}.
\begin{lemma}[Total update time]
\label{lem:incremental-amortize}
The total update time of Algorithm \ref{algo:lazy-qvi} is at most $\wt{O}(TH^5/\eps^3)$ over a sequence of $T$ action insertions.
\end{lemma}
\begin{proof}
For each new action $(s_h, a_h)$, the construction of $\wh{P}_{h}(s_h, a_h)$ takes $\wt{O}(N) = \wt{O}(H^3/\eps^2)$ time. The major overhead comes from the {\sc Propagate} part.
First, note the propagated $V$-value $\wt{V}_h(s_h)$ of a state $s_h$ can change at most $H /(\eps/4H) = O(H^2/\eps)$ times.
Next, for each state-action pair $(s_h, a_h)$, the $Q$-value $\wh{Q}_h(s_h, a_h) = r_{h}(s_h, a_h) + \E_{s_{h+1}\sim \wh{P}_{h}(s_h, a_h)} \wt{V}_{h+1}(s_{h+1})$  can change at most $O(N \cdot H^2/\eps) = \wt{O}(H^5/\eps^3)$ times, because the support of $\wh{P}_{h}(s_h, a_h)$ has size at most $N$, and each estimate $\wt{V}_{h+1}(s_{h+1})$ changes at $O(H^2/\eps)$ times as stated above.
The total number of state-action pair is bounded by $T$. We conclude the proof.
\end{proof}

The proof of Theorem \ref{thm:incremental} follows directly from Lemma \ref{lem:incremental-value-approx} and Lemma \ref{lem:incremental-amortize}.

We next prove the lower bound.

\begin{proof}[Proof of Theorem \ref{thm:incremental-lower}]
Let $n = T^{1-o(1)}$, $m = n^{o(1)}$. We reduce from $\mip$ with sets $B_1,\ldots, B_n$ and $C_1, \ldots, C_n$ defined over ground element $[m]$.
The MDP contains $H = 3$ steps. There is one single initial state $s_1$ at the first step $h= 1$. 
In the second step ($h=2$), there are $m = n^{o(1)}$ states $s_{2,1}, \ldots, s_{2,m}$,
and at the last step ($h=3$), there are two states $s_{3,1}$, $s_{3,2}$.

The sequence of new actions is as follow.
There is one action $a_{3}$ for the last step, and the reward satisfies $r_3(s_{3,1}, a_3) = 1$ and $r_3(s_{3,1}, a_3) = 0$, i.e., the reward is $1$ for $s_{3,1}$ and $0$ for $s_{3,2}$. 
There are $n$ actions $a_{1,1}, \ldots, a_{1, n}$ for the initial state at the first step, and we have
\[
P_{1}(s_{1}, a_{1, i}) = \unif(\{s_{2, k}: k \in B_{i}\}) \quad \text{and} \quad  r_{1}(s_{1}, a_{1, i}) = 0 \quad \forall i \in [n].
\]
The rest sequence divides into $n$ epochs, and in the $j$-th epoch ($j\in [n]$), there is one new action $a_{2, j}$ for each state $\{s_{2, k}\}_{k\in [m]}$.
Let $\delta = 1/4n$. At the end of $j$-th epoch, $t(j) \in [T]$, the transition and the reward of the new action $a_{2, j}$ satisfies
\[
P_{2}(s_{k}, a_{2, j}) = \left\{
\begin{matrix}
(\frac{j}{n+1} + \delta, 1 - \frac{j}{n+1} - \delta) & k \in C_j \\
(\frac{j}{n+1}, 1 - \frac{j}{n+1}) & k \notin C_j
\end{matrix}\right. \quad \text{and} \quad r_{2}(s_{k}, a_{2, j}) = 0. 
\]
In summary, the total number of state-action pairs at the end is $2 + n + mn = n^{1+o(1)} = T$.

First, a simple observation on the value function 
\begin{lemma}
At the end of epoch $j \in [n]$, the optimal policy satisfies
\begin{itemize}
\item $V^{*}_{t(j)}(s_{3,1})=1$ and $V^{*}_{t(j)}(s_{3,1})=0$
\item $V^{*}_{t(j)}(s_{2,k})= \frac{j}{n+1} + \delta$ when $k \in C_{j}$ and $V^{*}_{t(j)}(s_{2,k})= \frac{j}{n+1}$ otherwise
\item $V^{*}_{t(j)}(s_{1}) = \frac{j}{n+1} + \kappa_j\cdot \delta $, where $\kappa_j = \max_{i \in [n]}\frac{|C_j\cap B_i|}{b}$
\end{itemize}
\end{lemma}
\begin{proof}
The first claim is trivial. 
The second claim holds since the action $a_{2, j}$ is the optimal choice by the end of epoch $j$, and its value equals $Q^{*}_{t(j)}(s_{2,k}, a_{j}) = \frac{j}{n+1} + \delta$ when $k \in C_j$ and $Q^{*}_{t(j)}(s_{2,k}, a_{j}) = \frac{j}{n+1}$ when $k\notin C_j$.
For the last claim, for any $i \in [n]$, we have 
\begin{align*}
Q^{*}_{t(j)}(s_{1}, a_{1, i}) = &~ \sum_{k \in [m]} \Pr[s_{2} = s_{2, k}] \cdot V^{*}_{t(j)}(s_{2,k}) \\
= &~ \sum_{k \in C_{j}} \Pr[s_{2} = s_{2, k}] \cdot V^{*}_{t(j)}(s_{2,k}) + \sum_{k \in [m]\backslash C_{j}} \Pr[s_{2} = s_{2, k}] \cdot V^{*}_{t(j)}(s_{2,k}) \\
= &~ \frac{|B_{i} \cap C_{j}|}{b} \cdot \left(\frac{j}{n+1} +\delta\right)  + \left(1- \frac{|B_{i} \cap C_{j}|}{b}\right)\cdot \frac{j}{n+1} \\
= &~ \frac{j}{n+1} + \delta \cdot \frac{|B_{i} \cap C_{j}|}{b}.
\end{align*}
Taking maximum over $i \in [n]$, we have
$V^{*}_{t(j)}(s_{1}) = \frac{j}{n+1} + \kappa_j\cdot \delta$.
\end{proof}

Hence, any algorithm that returns $\delta / b = O(1/mn) = O(1/T)$ approximation to optimal $V$-value could distinguish between YES/NO instance of $\mip$, and therefore, assuming $\SETH$ is true, there is no algorithm with $n^{2-o(1)}/T = T^{1-o(1)}$ amortized runtime per update.
\end{proof}

\end{document}